\documentclass[journal]{IEEEtran}
\ifCLASSINFOpdf
   \usepackage[pdftex]{graphicx}
\else
\fi
\ifCLASSOPTIONcompsoc
  \usepackage[caption=false,font=normalsize,labelfont=sf,textfont=sf]{subfig}
\else
  \usepackage[caption=false,font=footnotesize]{subfig}
\fi
\usepackage{amsfonts}
\usepackage{amstext}
\usepackage{amsmath}
\usepackage{amsthm}
\usepackage{mathtools}
\usepackage{mathrsfs}
\usepackage[justification=centering]{caption}

\newtheorem{theorem}{Theorem}
\begin{document}
%
\title{Multi-Modal Coreference Resolution with the Correlation between Space Structures}
%
%

\author{Qibin~Zheng,
	Xingchun~Diao,
	Jianjun~Cao,
	Xiaolei~Zhou,
	Yi~Liu,
	and~Hongmei~Li
	\thanks{Q~Zheng, X~Diao, Y~Liu and H~Li are with Institute of Command and Control Engineering, Army Engineering University of PLA, Nanjing, China. E-mail: zqb1990@hotmail.com}
	\thanks{J~Cao and X~Zhou are with The 63rd Institute, National University of Defense Technology, Nanjing, China. E-mail: jianjuncao@yeah.net}
\thanks{}}

%
%

\markboth{}%
{Qibin \MakeLowercase{\textit{et al.}}: }
%



\maketitle

\begin{abstract}
Multi-modal data is becoming more common in big data background. Finding the semantically similar objects from different modality is one of the heart problems of multi-modal learning. Most of the current methods try to learn the inter-modal correlation with extrinsic supervised information, while intrinsic structural information of each modality is neglected. The performance of these methods heavily depends on the richness of training samples. However, obtaining the multi-modal training samples is still a labor and cost intensive work. In this paper, we bring a extrinsic correlation between the space structures of each modalities in coreference resolution. With this correlation, a semi-supervised learning model for multi-modal coreference resolution is proposed. We firstly extract high-level features of images and text, then compute the distances of each object from some reference points to build the space structure of each modality. With a shared reference point set, the space structures of each modality are correlated. We employ the correlation to build a commonly shared space that the semantic distance between multi-modal objects can be computed directly. The experiments on two multi-modal datasets show that our model performs better than the existing methods with insufficient training data.
\end{abstract}


%
\IEEEpeerreviewmaketitle

\section{Introduction}\label{intro}

\IEEEPARstart{R}{ecent} years have witnessed a surge of need in jointly analyzing of multi-modal data. As one of the fundamental problem of multi-modal learning, multi-modal coreference resolution aims to find the similar objects from different modality. It is widely used in such area as cross-modal retrieval, image captioning, cross-modal entity resolution. The problem is challenging because it requires detailed knowledge of different modalities and the correspondence between them\cite{Karpathy:2014}.

Most current approaches presume that there is a linear or non-linear projection between multi-modal data\cite{Rasiwasia:2010}\cite{Luo:2017}. These methods focus on how to utilize extrinsic supervised information to project one modality to the other or map both two modalities into a commonly shared space\cite{Rasiwasia:2010}\cite{Ngiam:2011}\cite{Socher:2013}\cite{Karpathy:2014}. The performance of these methods heavily depends on the richness of training samples. However, in real-world applications, obtaining the matched data from multiple modalities is costly and even impossible\cite{Gao:2018}. Therefor, it is urgently needed to develop a sample-insensitive methods for multi-modal coreference resolution.

When training samples are not rich enough, the intrinsic structure information of each modality can be helpful. In this paper specifically, the space structure of each modality is employed for multi-modal coreference resolution. For two set of cross-modal objects, although the representations are heterogeneous, the proximity relationship of them are similar. Taking the coreference resolution between images and textual descriptions as a example, Fig. \ref{fig:showcase} demonstrates the core idea of our method: given a set of images and one of corresponding textual descriptions, if two images(like two images of the bird) are similar to each other semantically, their corresponding descriptions should have similar meaning; otherwise, their meaning should be distinct. In other words, with proper distance metrics, the distance between images are similar to that between corresponding textual descriptions. Ignoring the scale difference, although the representation space of different modalities are heterogeneous, the space structure of them are similar. Hence, with a few matched pairs(the red points with same labels in the figure, called reference points), the space structure of different modalities can be anchored. 

Nevertheless, the raw feature spaces do not perform well for describing the space structure\cite{Rohrbach:2013}, higher-level semantic embedding spaces should be employed. Besides, it does not always hold that more matched objects bring higher resolution precision. It is significant to select reference points to describe space structure better.


\begin{figure*}
	
	\centering
	\includegraphics[width=10cm]{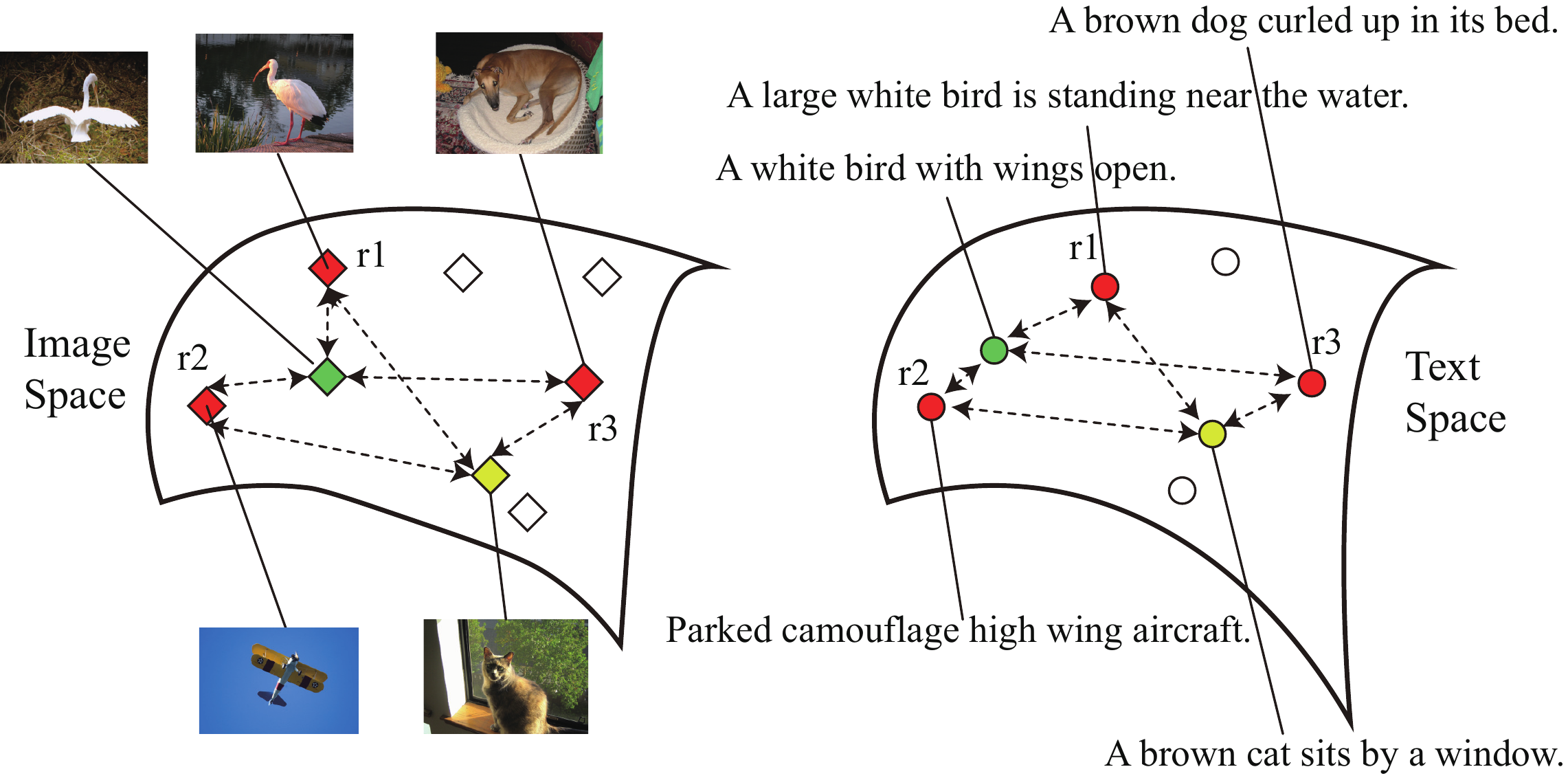}
	\caption{The core idea of proposed method}\label{fig:showcase}
\end{figure*}

%
%

Unlike previous works, we focus on digging the intrinsic correlation between modalities. In this paper, we bring intro-modal space structure information in resolving multi-modal coreference relation. The main contribution of this paper are listed as follows:

1) We investigate the intrinsic correlation between multi-modal data, and employ it to find semantically equivalent multi-modal objects. Different from similar works in zero-shot or few-shot learning\cite{Socher:2013b}\cite{Guo:2016}, we proved that the correlation between space structure widely exists among multi-modal dataset, rather than to learn it with a mount of training data every time.

2) To describe the space structure better, we utilize high level features to represent images and text, and an optimized strategy to select representative reference points.

3) Experiments on public data sets are carried to test the performance of our proposed methods comparing with state-of-the-art methods.

The paper is organized as follows. Section \ref{relatedwork} discusses the related work on multi-modal and cross-modal modeling. Section \ref{approach} introduces the proposed uniform representation of multi-modal data, and employ it for coreference resolution. Section \ref{experiment} tests the proposed method through the experiments on public datasets.
\section{Related work}\label{relatedwork}
According to the way of building the space, current methods for multi-modal coreference resolution can be divided into four types: 
\subsection{CCA-Based Methods}\label{CCA}
To the best of our knowledge, the first well-known cross-modal correlating model may be the CCA based model proposed by Hardoon et. al\cite{Hardoon:2004}. It learnt a linear projection to maximize the correlation between the representation of different modality in the projected space. Inspired by this work, many CCA based models are designed for cross-modal analyzing\cite{Rasiwasia:2010}\cite{Sharma:2012}\cite{Mroueh:2015}\cite{Wang:2017}. Rasiwasia et al.\cite{Rasiwasia:2010} utilized CCA to learn two maximally correlated subspaces, and multiclass logistic regression was performed within them to produce the semantic spaces respectively. Mroueh et al.\cite{Mroueh:2015} proposed a Struncated-SVD based algorithms to compute the full regularization path of CCA for multi-modal retrieval efficiently. Wang et al.\cite{Wang:2017} developed a new hypergraph-based Canonical Correlation Analysis(HCCA) to project low-level features into a shard space where intra-pair and inter-pair correlation be maintained simultaneously. Heterogeneous high-order relationship was used to discover the structure of cross-modal data.
\subsection{Deep Learning Methods}
Due to the strong learning ability of deep neural network, many deep model have been proposed for multi-modal analysis, such as\cite{Ngiam:2011}\cite{Srivastava:2012}\cite{Andrew:2013}\cite{Karpathy:2014}\cite{Frome:2013}\cite{Jiang:2017}\cite{Wei:2017}. Ngiam et al.\cite{Ngiam:2011} presented an auto-encoder model to learn joint representations for speech audios and videos of the lip movements. Srivastava and Salakhutdinov\cite{Srivastava:2012} employed restricted Boltzmann machine to learn a shared space between data of different modalities. Frome et al.\cite{Frome:2013} proposed a deep visual-semantic embedding(DeViSE) model to identify the visual objects using the information from labeled image and unannotated text. Andrew et al.\cite{Andrew:2013} introduced Deep Canonical Correlation Analysis to learn such nonlinear mapping between two views of data that the corresponding objects are linearly related in the representation space. Jiang et al.\cite{Jiang:2017} proposed a real time Internet cross-media retrieval method, in which deep learning was employed for feature extraction and distance detection. Due to the powerful representing ability of the convolutional neural network visual feature, Wei et al.\cite{Wei:2017} employed it coupled with a deep semantic matching method for cross-modal retrieval.

\subsection{Topic Model Methods}
Topic model is also helpful for uniform representing of multi-modal data, assuming that objects of different modality share some latent topics. Latent Dirichlet Allocation(LDA) based methods establish the shared space through the joint distribution of multi-modal data and the conditional relation between them\cite{Jia:2011}\cite{Roller:2013}. Roller and Walde\cite{Roller:2013} integrated visual features into LDA, and presented a multi-modal LDA model to learn joint representations for textual and visual data. Wang et al.\cite{Wang:2014} proposed multimodal mutual topic reinforce model(M$^3$R) to discover mutual consistent topics.

\subsection{Hashing-Based Methods}
For the rapid growth of data volume, the cost of finding nearest neighbor cannot be dismissed. Hashing is a scalable method for finding nearest neighbors approximately\cite{Wang:2010}. It projects data into Hamming space, where the neighbor search can be performed efficiently. To improve the efficient of finding similar multi-modal objects, many cross-modal hashing methods have been proposed\cite{Wang:2010}\cite{Kumar:2011}\cite{Zhen:2012}\cite{Ou:2013}\cite{Wu:2014}. Kumar and Udupa\cite{Kumar:2011} proposed a cross view hashing method to generate such hash codes that minimized the distance in Hamming space between similar objects and maximized that between dissimilar ones. Zhen et al.\cite{Zhen:2012} used a co-regularization framework to generate such binary code that the hash codes from different modality were consistent. Ou et al.\cite{Ou:2013} constructed a Hamming space for each modality and build the mapping between them with logistic regression. Wu et al.\cite{Wu:2014} proposed a sparse multi-modal hashing method for cross-modal retrieval.   

Besides these methods above, there still are other models proposed for multi-modal problems. Chen et al.\cite{Chen:2012} employed voxel-based multi-modal partial least square(PLS) to analyze the correlations between FDG PET glucose uptake-MRI gray matter volume scores and apolipoprotein E epsilon 4 gene dose in cognitively normal adults. 

Although these methods have achieved great success in multi-modal learning, most of them need a mass of training data to learn the complex correlation between objects from different modality. To reduce the demand of training data, Gao et al.\cite{Gao:2018} proposed an active similarity learning model for cross-modal data. Nevertheless, without extra information, the improvement is limited.

\begin{figure*}
	\centering
	\includegraphics[width=10cm]{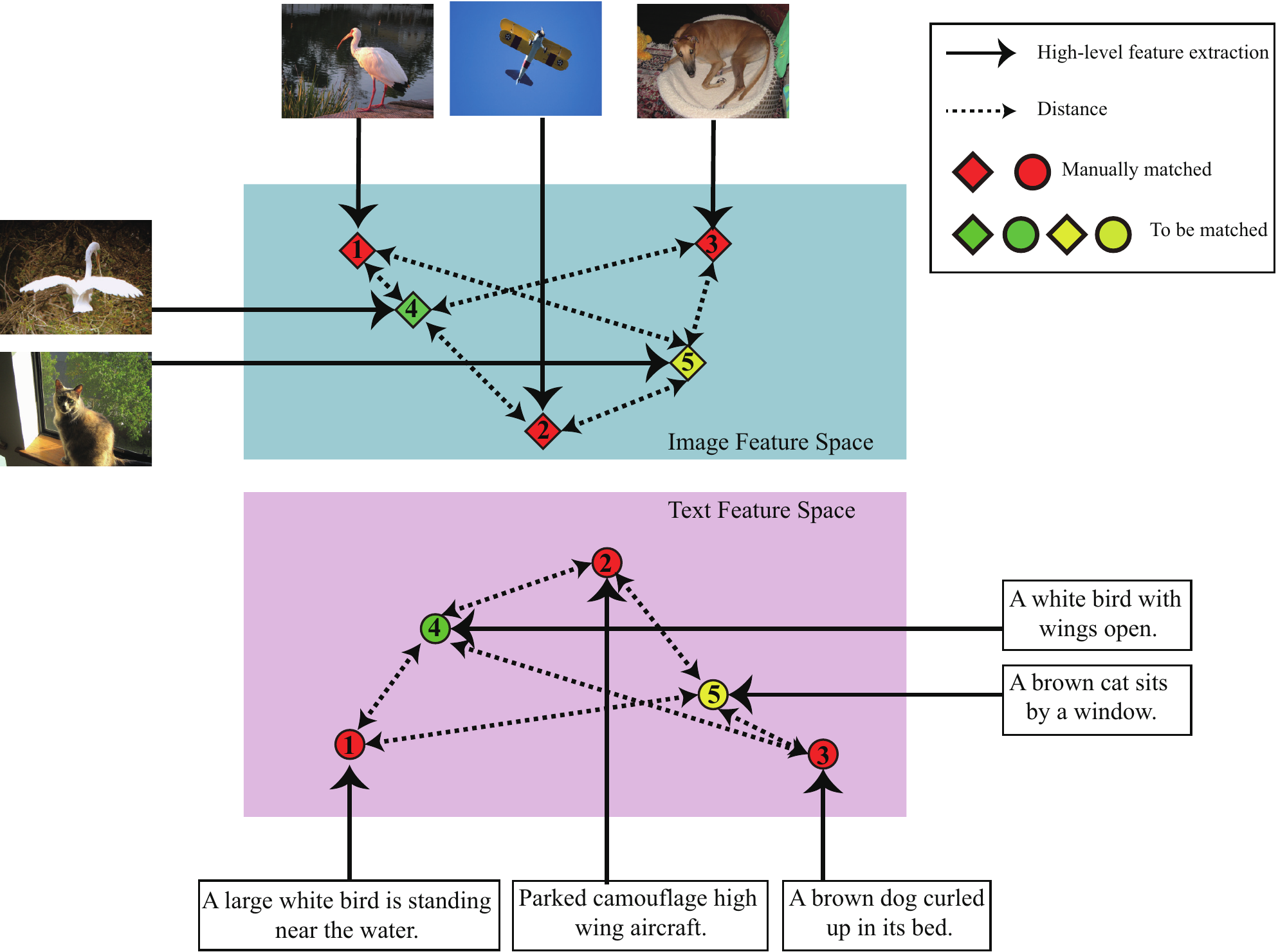}
	\caption{Overview of our multi-modal coreference resolution model.}\label{fig:framework}
\end{figure*}
\section{Proposed Approach}\label{approach}
To perform the coreference resolution with insufficient matched samples, we take the similarity between space structures into consideration. Given two datasets $X=\{x_1, x_2, \cdots, x_n\}\in \mathbb{R}^{n\times d}$ and $Y=\{y_1, y_2, \cdots, y_m\}\in \mathbb{R}^{m\times e}$ from two modalities, Fig \ref{fig:framework} shows the overview of our proposed model, the red points refer to the matched objects in two modality, the solid lines refer to semantic embedding, the dotted lines mean the distance between two objects. $X$ and $Y$ are firstly mapped into semantic spaces(the blue and pink area in the figure) respectively as in \ref{semantic}: 
\begin{equation}\label{eq:m1}
\mathscr{M}_1:X \rightarrow \mathcal{X},
\end{equation}
and
\begin{equation}\label{eq:m2}
\mathscr{M}_2:Y \rightarrow \mathcal{Y}.
\end{equation}
Secondly, represent the to-be-matched objects(the non-red points in Fig \ref{fig:framework}) of $\mathcal{X}$ and $\mathcal{Y}$(the red ones in the figure) with the space structure based scheme in Section \ref{ssr}:
\begin{equation}\label{eq:m3}
\mathscr{M}_3:\mathcal{X}\rightarrow \Re ^\mathcal{X},
\end{equation}
and
\begin{equation}\label{eq:m4}
\mathscr{M}_4:\mathcal{Y}\rightarrow \Re ^\mathcal{Y}.
\end{equation}

Since $\Re ^\mathcal{X}$ and $\Re ^\mathcal{Y}$ are positively correlated(proved in Section \ref{theory}), isomorphic and homogeneous, they can be linearly projected into each other:
\begin{equation}
\label{eq:m5}
\mathscr{M}_5:\Re ^\mathcal{X} \leftrightarrow  \Re ^\mathcal{Y}.
\end{equation}
In $\Re ^\mathcal{XY}$, cross-modal similarities are measured with cosine distance. With the similarity measurement, multi-modal coreference resolution can be solved by traditional methods.  

\subsection{High level-feature extraction for text and images}\label{semantic}
In general, different low-level representations are adopted for different modalities, and there is no explicit correspondence between them\cite{Costa:2014}\cite{Huang:2017}. Similar to the hypothesis fo zero-shot learning\cite{Socher:2013b}, we believe that the space structures of different modalities are similar, and try to connect modalities with them. It is crucial whether the pair-wise distances over the representation can reflect the semantic similarity properly.

The distances over raw or low-level features are always far from the real semantic ones. There are many works available to achieve better representations for various of data,  like\cite{Karpathy:2014}\cite{Socher:2013}\cite{Wei:2017}\cite{Frome:2013}. In this paper, we utilize the methods below for feature exacting of text and image specifically:

\subsubsection{Smooth inverse frequency for text embedding}
Smooth inverse frequency(SIF) is a simple but powerful method for measuring textual similarity\cite{Arora:2017}. It achieves great performance on a variety of textual similarity tasks, and even beats some sophisticated supervised methods including RNN and LSTM. In addition, the method is very suitable for domain adaptive setting. The core idea of SIF method is surprisingly simple: compute the weighted average of the word vectors and remove the first principle projection of the average vectors. For the superiority over performance, simplicity and domain adaption of SIF methods, it is utilized for text embedding in this work. 

As in \cite{Arora:2017}, we employ SIF methods on 300-dimensional GloVe word vectors trained on 840 billion token Common Crawl corpus for textual semantic embedding, i.e., the mapping in Equation(\ref{eq:m1}).

\subsubsection{Convolutional neural network off-the-shelf for image embedding}
Convolutional neural network(CNN) has demonstrated outstanding ability for machine learning tasks of images, like image classification and object detection. Wei et. al proposed to utilize CNN visual features for cross-modal retrieval\cite{Wei:2017}, which performs better than commonly used features and achieves a baseline for cross-modal retrieval. Firstly, off-the-shelf CNN features\cite{Donahue:2014} are extracted from the CNN model pretrained on ImageNet; then a fine-tuning step is processed for the target dataset.

In this paper, for better semantic representation, we extracted CNN features from images with the method above, namely, the mapping in Equation(\ref{eq:m2}). 

\subsection{Representation scheme of the space structure of single modality}\label{ssr}
Although multi-modal data are highly heterogeneous, most of them are represented in Euclidean or Hamming space. In this section we modify the representation scheme of space structure in\cite{Qian:2016}, and employ it to represent the objects from Euclidean and Hamming space.
\begin{figure*}
	\centering
	\begin{tabular}{cc}
		\fbox{\includegraphics[width=5cm]{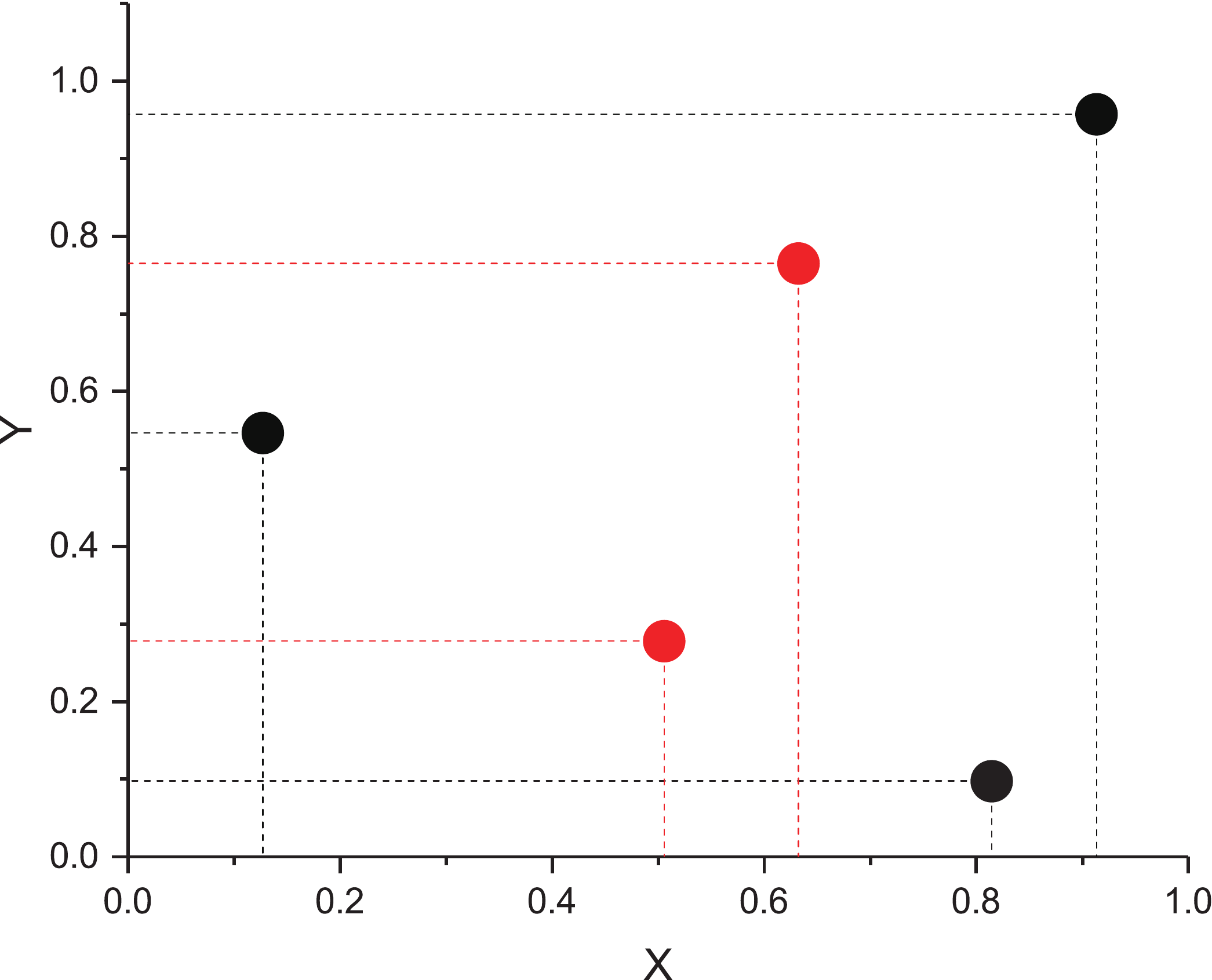}}&
		\fbox{\includegraphics[width=5cm]{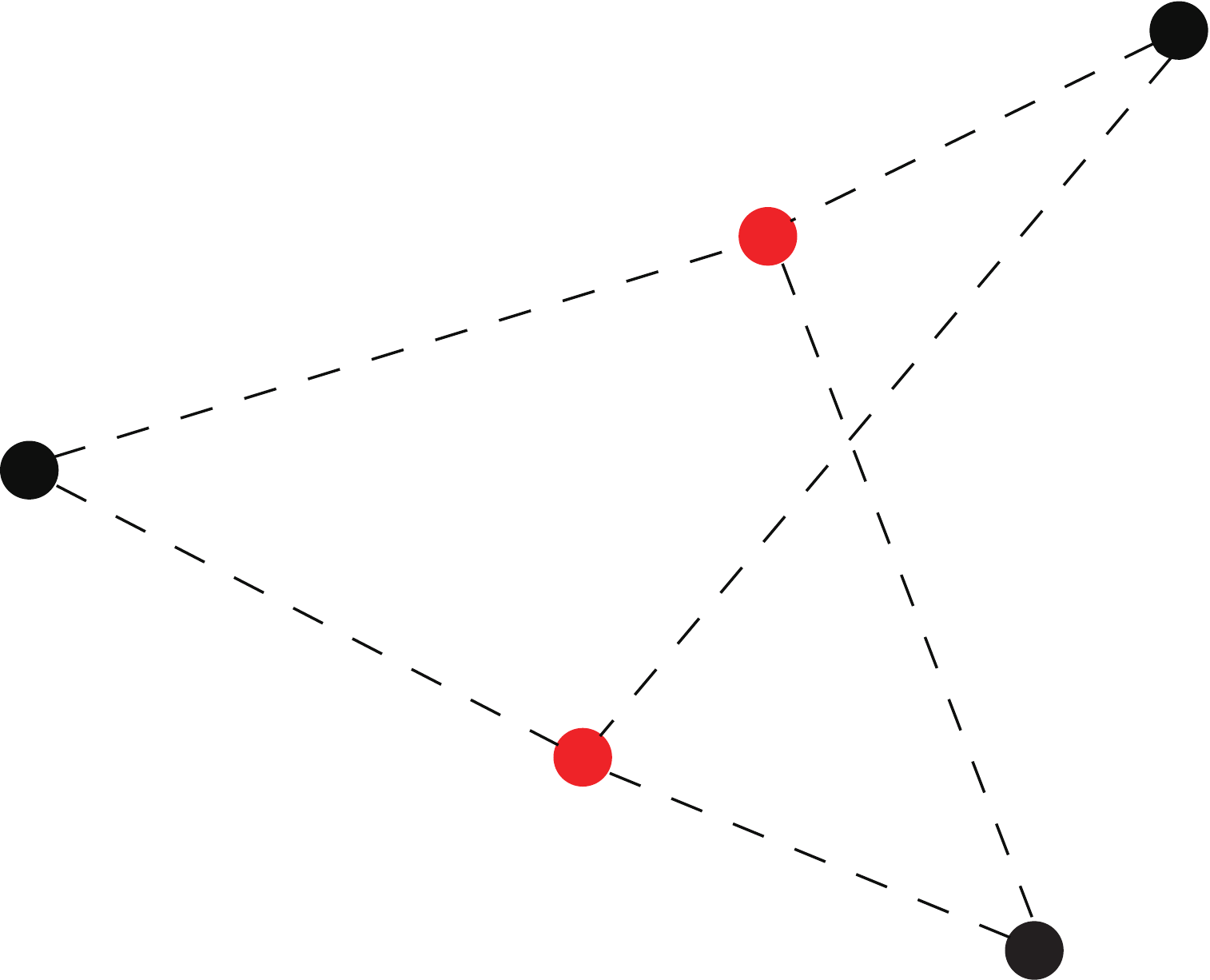}}\\
		(a)&(b)
	\end{tabular}
	\caption{The position of five points in (a) can be described by the distances to the reference points as in (b). }\label{fig:reference}
\end{figure*}
Qian et al. proposed a space structure based representation scheme for clustering of categorical data\cite{Qian:2016}. The space structure is figured with the pair-wise similarity: given a categorical object $c_i$, it is re-represented as $r_i^{c}=\{r_{i1}^{c}, r_{i2}^{c}, \cdots, r_{ij}^{c},\cdots, r_{in}^{c}\}$, where $r_{ij}$ is the similarity between $c_i$ and $c_j$, $n$ is the number of the categorical dataset $O$. Qian et al. utilized all the members in the dataset as the reference system to represent categorical data, however, it is unnecessary and inefficient to do so in Euclidean or Hamming space.

In Euclidean space, an object is represented as its distances from the coordinate axes (as Fig. \ref{fig:reference} (a)). However, it can also be positioned by its Euclidean distance from several given points(called reference points). For example, in Fig. \ref{fig:reference} (b), the red points can be uniquely positioned by their distances from the black ones. In fact, for a $m$-dimensional object in Euclidean space, $m+1$ non-colinear reference points are enough to locate it. Based on the representation scheme, a set of numeric object $\mathcal{X}=\{x_1, x_2, \cdots, x_n\}$ can be mapped into $\Re ^\mathcal{X}=\{r_{1}^{x}, r_{2}^{x}, \cdots, r_{n}^{x} \}$ with a reference point set $O^\mathcal{X}=\{o_1^{x}, o_2^{x}, \cdots, o_j^{x}, \cdots, o_k^{x}\} \subset \mathcal{X}$ ($m+1\le k=|O_\mathcal{X}| \ll |\mathcal{X}|$), where
\begin{equation}
r_{i}^{x}=(r_{i1}^{x}, r_{i2}^{x},\cdots, r_{ij}^{x}, \cdots, r_{ik}^{x}),
\end{equation}

\begin{equation}
\label{euclidean}
r_{ij}^{x}=\sqrt{\sum_{l=1}^m{\left( x_{il}-o_{jl}^{x} \right) ^2}}.
\end{equation}

Similarly, a binary object in Hamming space can be also represented by its Hamming distances from the reference points. For a $m$-dimensional object, $m$ disparate ones as reference points are enough to figure out its position in Hamming space. Given a reference point set $O^\mathcal{Y}=\{o_1^{y}, o_2^{y}, \cdots, o_j^{y}, \cdots, o_k^{y}\} \subset \mathcal{X}$ ($m\le k=|O_\mathcal{Y}| \ll |\mathcal{Y}|$), a set of binary object $\mathcal{Y}=\{y_1, y_2, \cdots, y_n\}$ in Euclidean space is mapped into $\Re ^\mathcal{Y}=\{r_{1}^{y}, r_{2}^{y}, \cdots, r_{n}^{y} \}$ where
\begin{equation}
r_{i}^{y}=(r_{i1}^{y}, r_{i2}^{y},\cdots, r_{ij}^{y}, \cdots, r_{ik}^{y}),
\end{equation}
 
\begin{equation}
\label{hamming}
r_{ij}^{y}=\sum_{l=1}^m{\left( y_{il}\oplus o_{jl}^{y} \right) },
\end{equation}
$\oplus$ is xor operator.

\begin{figure}[htbp]
	\centering
	\includegraphics[width=8cm]{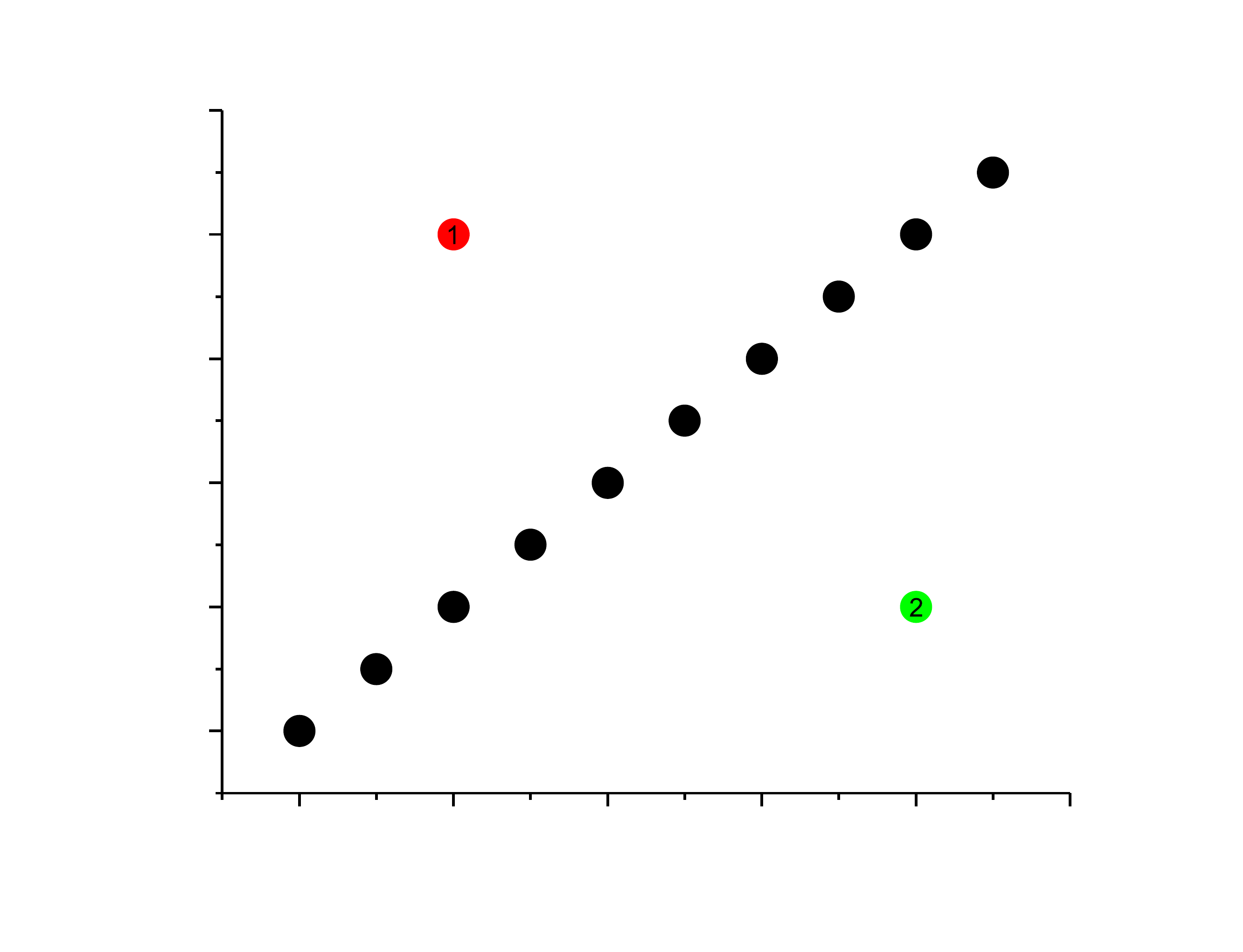}
	\caption{Reference points}\label{fig:referencePoint}
\end{figure}

Although more matched multi-modal objects would benefit coreference resolution better intuitively, some pairs of objects may undermine the representing ability of the space structure. As Fig \ref{fig:referencePoint} illustrates, the red and green points being taken as reference points are definitely better than the black ones. On the one hand, A good reference set $R$ should adequately reflect the distribution of the dataset, then the members of reference set should be quite distinct with each other. On the other hand, for higher  discrimination, their distances from the non-reference objects should be distinct either. Thus, we transfer the selection of reference points to an optimization problem:
\begin{equation}
\label{eq:opt}
\max\text{\,\,}L\left( R|k \right) =\frac{\sum_{i\ne j}^k{d\left( o_i,o_j \right)}}{\left( \sum_{i=1}^k{\sigma \left( o_i \right)} \right) ^{-\lambda}},
\end{equation}
where $k\ge m+1$ is the number of reference points(user-specified), $\sigma(o_i)$ is the variance of the distances of $o_i$ from all non-reference objects, $\lambda > 0$ is a balance factor. 
 

Being represented with the scheme above, the similarity between numerical objects $x_i$ and $x_j$(or binary objects $y_i$ and $y_j$) can be directly measured with Euclidean or cosine distance between them as\cite{Qian:2016}: 

\begin{equation}
\label{eq:eucilidean}
dist(r_i,r_j)=\sqrt{\sum_l^k{\left( r_{il}-r_{jl} \right) ^2}},
\end{equation}
or
\begin{equation}
\label{eq:hamming}
dist(r_i,r_j)=1-\frac{r_i\cdot r_{j}^{T}}{\lVert r_i \rVert \times \lVert r_j \rVert}.
\end{equation}


\subsection{Correlation between the space structure of different modalities}\label{theory}
Although the feature spaces of multi-modal data are usually highly heterogeneous, they may share some latent properties. In this part, we prove that the pair-wise similarities of different modalities are correlated to each other. Based on this, these modalities share similar space structures. 

As discussed in section \ref{CCA}, CCA based methods aim to discover the linear correlation between multi-modal data.  Firstly, we assume that $\mathcal{X}$ and $\mathcal{Y}$ are linearly correlated as
\begin{equation}
\label{linearC}
\mathcal{Y}=\mathcal{X}M
\end{equation}
where $M \in \mathbb{R}^{d\times e}$ is a projection matrix. Let $S_\mathcal{X}(i,j)$ be the similarity between $x_i$ and $x_j$, $S_\mathcal{Y}(i,j)$ be that between $y_i$ and $y_j$. We have 
\begin{theorem}
	\label{theorem1}
	If multi-modal data $\mathcal{X}$ and $\mathcal{Y}$ are linearly correlated as in Equation (\ref{linearC}), there exists a positive correlation between $S_\mathcal{X}(i,j)$ and $S_\mathcal{Y}(i,j) (i,j=1,2,\cdots,n)$.
\end{theorem}
\begin{proof}
	For simplicity, we assume all $x_{i,\varphi }$($i=1, 2, \cdots, n~and~\varphi =1, 2, \cdots, d$) subject to a standard normal distribution $N\left( 0,1 \right)$ and are independent to each other. The inner product in Equation(\ref{simX}) and (\ref{simY}) are utilized to compute the similarity matrix $S_\mathcal{X}(i,j)$ and $S_\mathcal{Y}(i,j)$).
	\begin{equation}
	\label{simX}
	S_\mathcal{X}(i,j)=x_i x_{j}^T
	\end{equation}
	\begin{equation}
	\label{simY}
	S_\mathcal{Y}(i,j)=y_i y_{j}^T=x_i M M^T x_{j}^T
	\end{equation}
	Since $MM^T$ is symmetric, then it can be diagonalized as 
	\begin{equation}
	S_\mathcal{Y}(i,j)=y_i y_{j}^T=x_i P \varLambda P^T x_{j}^T.
	\end{equation}
	The Pearson correlation coefficient between $S_\mathcal{X}(i,j)$ and $S_\mathcal{Y}(i,j)$ is
	
	\begin{equation}
	\label{rho}
	\rho _{S_\mathcal{X}S_\mathcal{Y}}=\frac{\text{Cov}\left( S_\mathcal{X}\left( i,j \right) ,S_\mathcal{Y}\left( i,j \right) \right)}{\sqrt{D\left( S_\mathcal{X}\left( i,j \right) \right) \cdot D\left( S_\mathcal{Y}\left( i,j \right) \right)}},
	\end{equation}
	where the variance of $S_\mathcal{X}(i,j)$ and $S_\mathcal{Y}(i,j)$ is 
	\begin{equation}
	\label{eq:dx}
	\begin{aligned}
	D\left( S_\mathcal{X}\left( i,j \right) \right) =D\left( x_ix_{j}^{T} \right) 
	=D\left( \sum_{\varphi =1}^d{x_{i\varphi}x_{j\varphi}} \right)\\
	=\sum_{\varphi =1}^d{D\left( x_{i\varphi}x_{j\varphi} \right) =d}
	\end{aligned}~~,
	\end{equation}
	and
	\begin{equation}
	\label{eq:dy}
	\begin{aligned}
	D\left( S_\mathcal{Y}\left( i,j \right) \right) &=D\left( y_iy_{j}^{T} \right) \\
	&=D\left( x_iP\varLambda P^Tx_{j}^{_T} \right)\\ 
	&=D\left( \sum_{\gamma =1}^d{\lambda _{\gamma}\left( x_ip_{\gamma}^{T} \right) \left( x_jp_{\gamma}^{T} \right)} \right)\\
	&=D\left( \sum_{\mu =1,\nu =1}^d{x_{i\mu}x_{j\nu}\sum_{\gamma =1}^d{\lambda _{\gamma}p_{\gamma \mu  }^Tp_{\gamma \nu}^T}} \right)
	\end{aligned}.
	\end{equation}
	
	Since 
	\begin{equation}
	\begin{aligned}
	Cov\left( x_{i\mu}x_{j\nu},x_{iu}x_{jv} \right)&=E\left( x_{i\mu}x_{j\nu}x_{iu}x_{jv} \right)\\
	&-E\left( x_{i\mu}x_{j\nu} \right) E\left( x_{iu}x_{jv} \right)\\ 
	&=E\left( x_{i\mu}x_{iu} \right) E\left( x_{j\nu}x_{jv} \right)\\
	&=0
	\end{aligned}~~,
	\end{equation}
	$x_{i\mu}x_{j\nu}$ and $x_{iu}x_{jv}$ are dependent with each other if $\mu \ne u$ or $\nu \ne v$, then
	\begin{equation}
	\label{eq:dy2}
	\begin{aligned}
	D\left( S_\mathcal{Y}\left( i,j \right) \right)&=\sum_{\mu =1,\nu =1}^d{\left( \sum_{\gamma =1}^d{\lambda _{\gamma}p_{\gamma \mu}^Tp_{\gamma \nu}}^T \right) ^2D\left( x_{iu}x_{jv} \right)}\\
	&=\sum_{\mu =1,\nu =1}^d{\left( \sum_{\gamma =1}^d{\lambda _{\gamma}p_{\gamma \mu}^Tp_{\gamma \nu}^T} \right) ^2}
	\end{aligned}~~.
	\end{equation}
	The covariance of $S_\mathcal{X}(i,j)$ and $S_\mathcal{Y}(i,j)$ is	\begin{equation}
	\label{eq:cov1}
	\begin{aligned}
	\text{Cov}\left( S_\mathcal{X}\left( i,j \right) ,S_\mathcal{Y}\left( i,j \right) \right) =\text{Cov}\left( x_ix_{j}^{T},x_iP\varLambda P^Tx_{j}^{T} \right)\\
	=\text{Cov}\left( x_ix_{j}^{T},\sum_{\gamma =1}^d{\lambda _{\gamma}\left( x_ip_{\gamma}^{T} \right) \left( x_jp_{\gamma}^{T} \right)} \right) \\
	=\sum_{\gamma =1}^d{\lambda _{\gamma}\text{Cov}\left( x_ix_{j}^{T},\left( x_ip_{\gamma}^{T} \right) \left( x_jp_{\gamma}^{T} \right) \right)}\\
	=\sum_{\gamma =1}^d{\lambda _{\gamma}\text{Cov}\left( \sum_{\varphi =1}^d{x_{i\varphi}x_{j\varphi}},\sum_{\mu =1\text{,}\nu =1}^d{p_{\gamma \mu}^Tp_{\gamma \nu}^Tx_{i\mu}x_{j\nu}} \right)}
	\end{aligned}.
	\end{equation}
	Since $x_{i,l}$ is independent of each other, then
	\begin{equation}
	Cov\left( \sum_{\varphi =1}^d{x_{i\varphi}x_{j\varphi}},\sum_{\mu \ne \nu}^d{p_{\gamma \mu}^Tp_{\gamma \nu}^Tx_{i\mu}x_{j\nu}} \right) =0,
	\end{equation}
	then
	\begin{equation}
	\begin{aligned}
	\text{Cov}\left( \sum_{\varphi =1}^d{x_{i\varphi}x_{j\varphi}},\sum_{\mu =1\text{,}\nu =1}^d{p_{\gamma \mu}^Tp_{\gamma \nu}^Tx_{i\mu}x_{j\nu}} \right)\\
	=\sum_{\varphi =1}^d{(p_{\gamma \varphi}^T)^{2}}\text{Cov}\left( x_{i\varphi}x_{j\varphi},x_{i\varphi}x_{j\varphi} \right) =\sum_{\varphi =1}^d{(p_{\gamma \varphi}^T)^2} 
	\end{aligned}.
	\end{equation}
	Finally,
	\begin{equation}
	\label{eq:cov}
	\text{Cov}\left( S_\mathcal{X}\left( i,j \right) ,S_\mathcal{Y}\left( i,j \right) \right) =\sum_{\gamma =1}^d{\lambda _{\gamma}\sum_{\varphi =1}^d{(p_{\gamma \varphi}^T)^2}}.
	\end{equation}
	
	From Equation (\ref{rho}), (\ref{eq:dx}), (\ref{eq:dy}) and (\ref{eq:cov}), the correlation coefficient is  
	\begin{equation}
	\rho _{S_\mathcal{X}S_\mathcal{Y}}=\frac{\sum_{\gamma =1}^d{\lambda _{\gamma}\sum_{\varphi =1}^d{(p_{\gamma \varphi}^T)^2}}}{\sqrt{d\sum_{\mu =1,\nu =1}^d{\left( \sum_{\gamma =1}^d{\lambda _{\gamma}p_{\gamma \mu}^Tp_{\gamma \nu}^T} \right) ^2}}}.
	\end{equation}
	Because $MM^T$ is a non-negative symmetric matrix, we have Equation(\ref{lamb}) unless it is a zero matrix either, which is unreasonable obviously.
	\begin{equation}
	\label{lamb}
	\sum_{\gamma =1}^d{\lambda _{\gamma}}=\sum_{\gamma =1}^d{m_{\gamma \gamma}}>0
	\end{equation}
	In Equation(\ref{lamb}), $mm_{\gamma\gamma}$ refer to the the principal diagonal elements of $MM^T$. In conclusion, there exist a positive correlation between $S_\mathcal{X}(i,j)$ and $S_\mathcal{Y}(i,j)$.
\end{proof}

In Theorem \ref{theorem1} we assume that $\mathcal{X}$ is linearly correlated to $\mathcal{Y}$, however, if the correlation is nonlinear the conclusion also hold. Similar to \cite{Luo:2017}, we define a nonlinear mapping
\begin{equation}
\label{nonlinearC}
\mathcal{Y}=\sigma \left( \mathcal{X}M + B\right),
\end{equation}
where $M$ is also a linear projection matrix, $B$ is a bias matrix, $\sigma$ is specified as Sigmoid function. The correlation between $S_\mathcal{X}$ and $S_\mathcal{Y}$ is similar to that in \textbf{Theorem \ref{theorem1}}.
\begin{theorem}
	\label{theorem2}
	If multi-modal data $\mathcal{X}$ and $\mathcal{Y}$ are non-linearly correlated as in Equation (\ref{nonlinearC}), there exists a positive correlation between $S_\mathcal{X}(i,j)$ and $S_\mathcal{Y}(i,j) (i,j=1,2,\cdots,n)$.
\end{theorem}

Same as \textbf{Theorem \ref{theorem1}}, we still use Equation (\ref{simX}) and (\ref{simY}) to compute $S_\mathcal{X}(i,j)$ and $S_\mathcal{Y}(i,j)$. Since $\mathcal{X}\mathcal{X}^T$ and $\mathcal{X}MM^T\mathcal{X}^T$ are correlated, and $\sigma$ is a monotonic increasing function, it is clear that $\mathcal{X}\mathcal{X}^T$is positively correlated to $\mathcal{Y}\mathcal{Y}^T$.

From \textbf{Theorem \ref{theorem1}} and \textbf{Theorem \ref{theorem2}}, whether $\mathcal{X}$ and $\mathcal{Y}$ are linearly or nonlinearly correlated to each other, $S_\mathcal{X}$ and $S_\mathcal{Y}$ are correlated to each other in most case. This kind of correlation would be utilized for measuring the similarity between cross-modal objects bellow. 

\subsection{Cross-modal projection and coreference resolution}\label{model}
Considering the positive correlation between $S_\mathcal{X}$ and $S_\mathcal{Y}$, since $\Re ^\mathcal{X}$ and $\Re ^\mathcal{Y}$ are the sub-matrices of $S_\mathcal{X}$ and $S_\mathcal{Y}$ respectively, they are positively correlated to each other, too. Taking a set of already matched objects as the shared reference set $O=\{(o_1^{x},o_1^{y}), \cdots, (o_i^{x},o_i^{y}), \cdots, (o_k^{x},o_k^{y}) \}$, $\Re ^\mathcal{X}$ and $\Re ^\mathcal{Y}$ would be isomorphic: the corresponding dimensions of $\Re ^\mathcal{X}$ and $\Re ^\mathcal{Y}$ are consistent in meaning and value\footnote{It should be noted that although the cardinality of reference set should be larger than the representation dimension to avoid confusion, it is difficult to achieve. In fact, slight confusion can reduce the possibility of missed detection. Here, we select a refernce set from all the matched objects with the strategy in Section \ref{ssr}}.

A linear mapping is performed on each dimension to eliminate the difference in the scales and distance metrics of $\Re ^\mathcal{X}$ and $\Re ^\mathcal{Y}$:
\begin{equation}
\Re ^\mathcal{X}M+B \rightarrow \Re ^\mathcal{Y},
\end{equation}
where $M$ is a diagonal matrix, $B$ is a bias matrix where each column is a constant. Since parameters of the projection are quite few, $M$ and $B$ can be learnt with a regression on the shared reference set $R$, the object is:
\begin{equation}
\min\,\sum_{\left( o_{i}^{x},o_{j}^{y} \right) \in O}{l\left( o_{i}^{x}M+b,o_{j}^{y} \right)}\,+\gamma g\left( M \right),
\end{equation}
where $l$ is the square loss function, $b\in B$ is the bias vector, $g$ is a regularization item.  

Although it is all right to take $\Re ^\mathcal{X}$ or $\Re ^\mathcal{Y}$ as the target space, visual feature space is more appropriate because it suffers less from the hubness problem\cite{Zhang:2017}. 

Giving an object $r_i^{x} \in \Re ^{\mathcal{X}}$, the $r_j^{y} \in \Re ^{\mathcal{Y}}$ that most closely matches it is that for which minimizes
\begin{equation}
dist(r_i^{x}M+b,r_j^{y})
\end{equation}

\section{Experiments}\label{experiment}
\subsection{Datasets}
We employ the following datasets to evaluate the performance of the proposed method: \textbf{Wikipedia}\cite{Rasiwasia:2010}, \textbf{Pascal-sentences}\cite{Rashtchian:2010}.

\textbf{Wikipedia}:This dataset contains 2866 pairs of images and text from ten categories. Each pair of image and text are extracted from Wikipedia's articles\cite{Rasiwasia:2010}. Instead of SIFT BoVW visual features and LDA textual features provided by\cite{Rasiwasia:2010}, we extract CNN fc6 features\footnote{CNN feature is extract with Caffe feature extraction tool.} and Sif features\footnote{Python codes can be obtained from: https://github.com/PrincetonML/SIF.} like in Section \ref{semantic}.

\textbf{Pascal-sentences}: This dataset is a subset of Pascal VOC, which contains 1000 pairs of image and corresponding textual description from twenty categories. Feature extraction are same as that of \textbf{Wikipedia} dataset.

\begin{figure*}[!ht]
	\centering
	\begin{tabular}{cc}
		\fbox{\includegraphics[width=7cm]{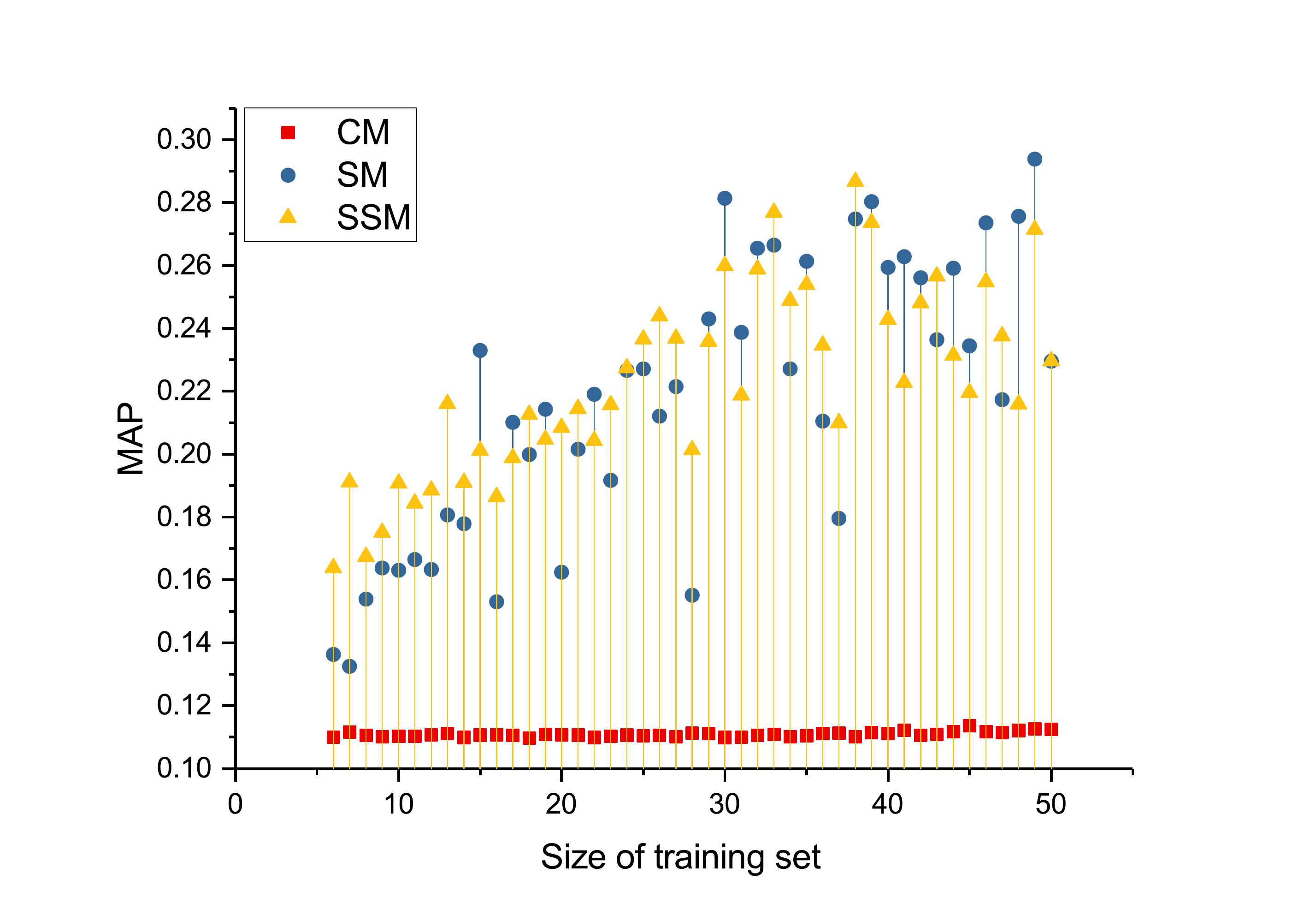}}&
		\fbox{\includegraphics[width=7cm]{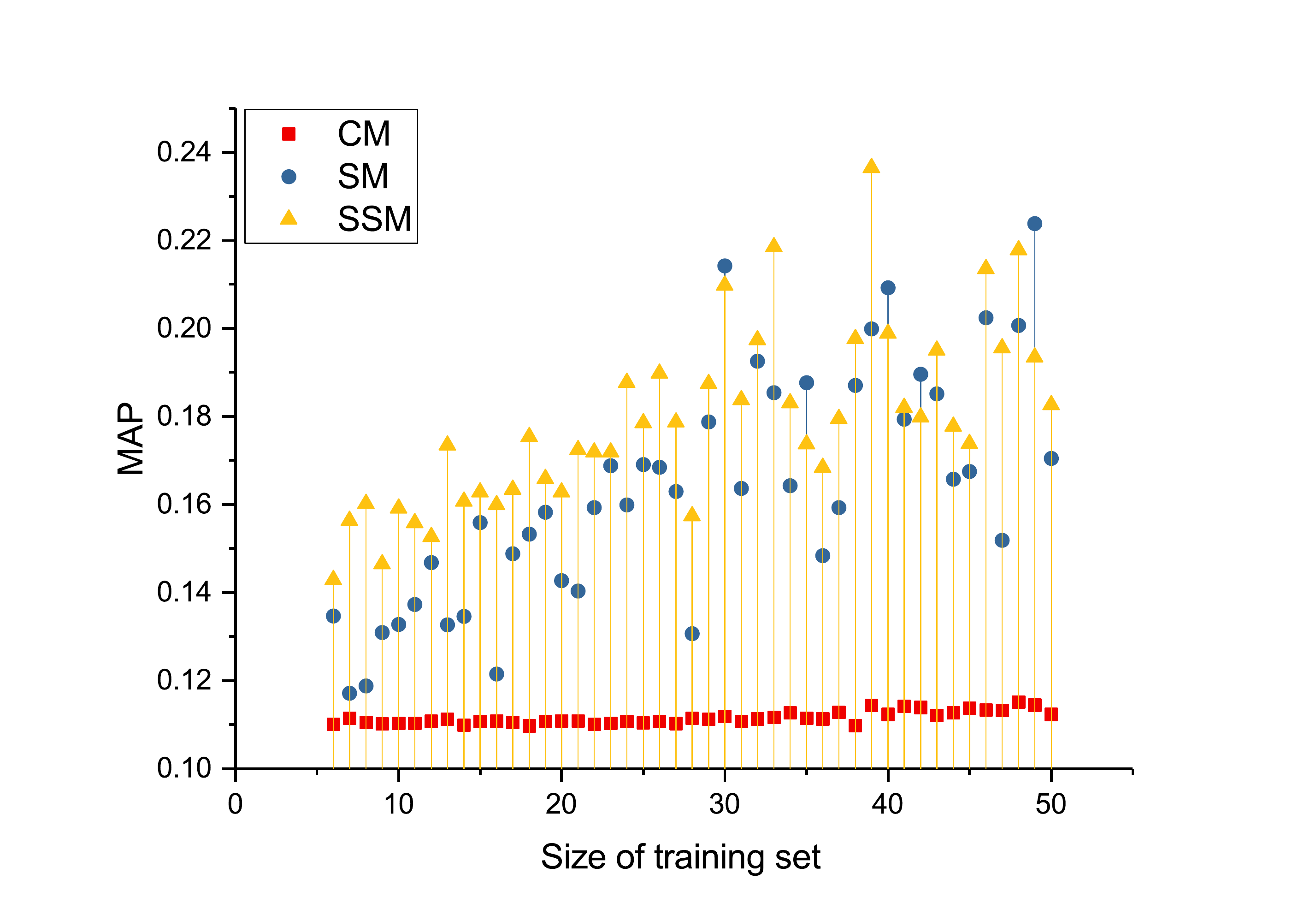}}\\
		(a)&(b)\\
	\end{tabular}
	\caption{The mAP score on Wikipedia dataset: (a) for image query, (b) for text query }\label{fig:wiki}
\end{figure*}
\begin{figure*}[!ht]
	\centering
	\begin{tabular}{cc}
		\fbox{\includegraphics[width=7cm]{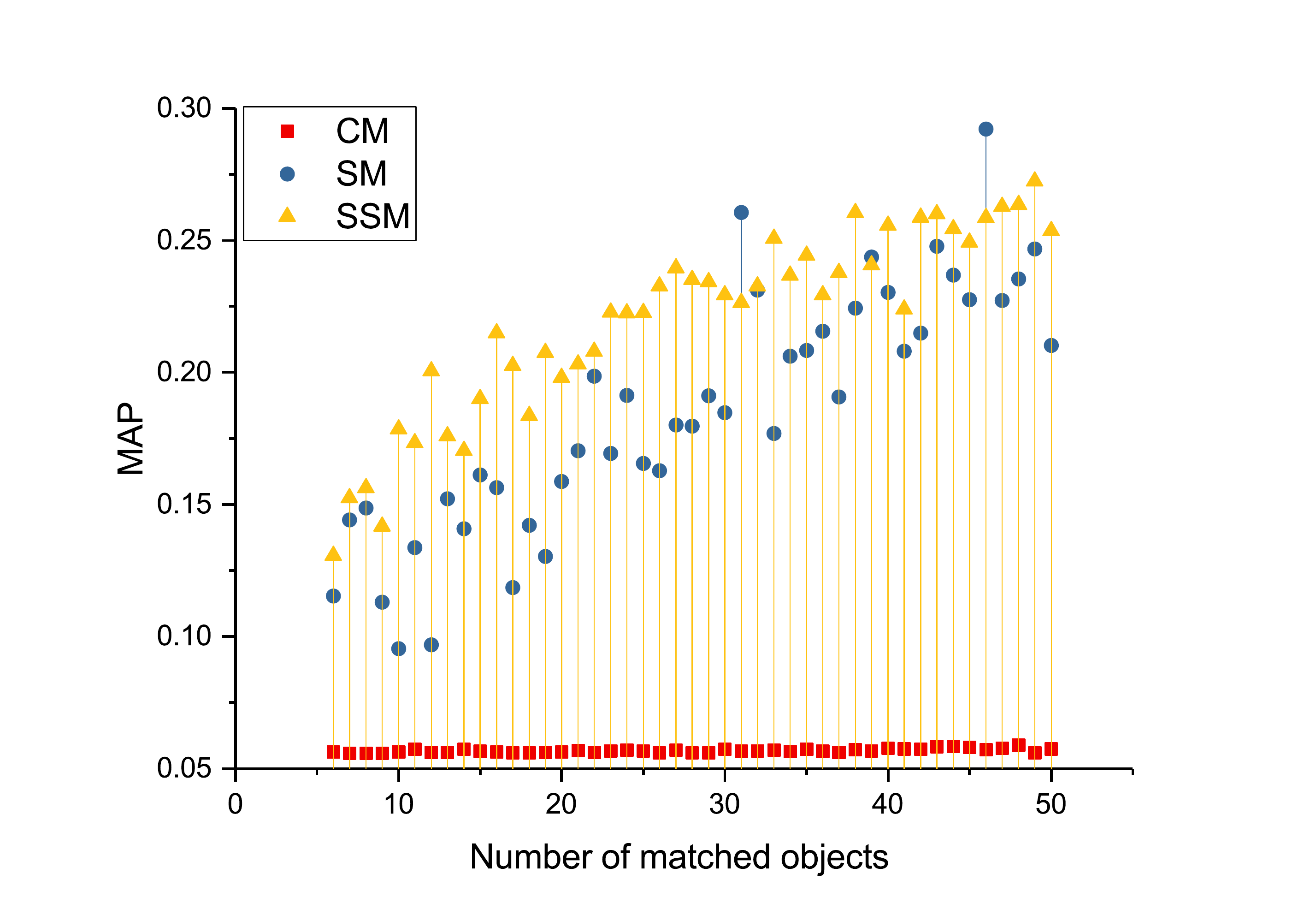}}&
		\fbox{\includegraphics[width=7cm]{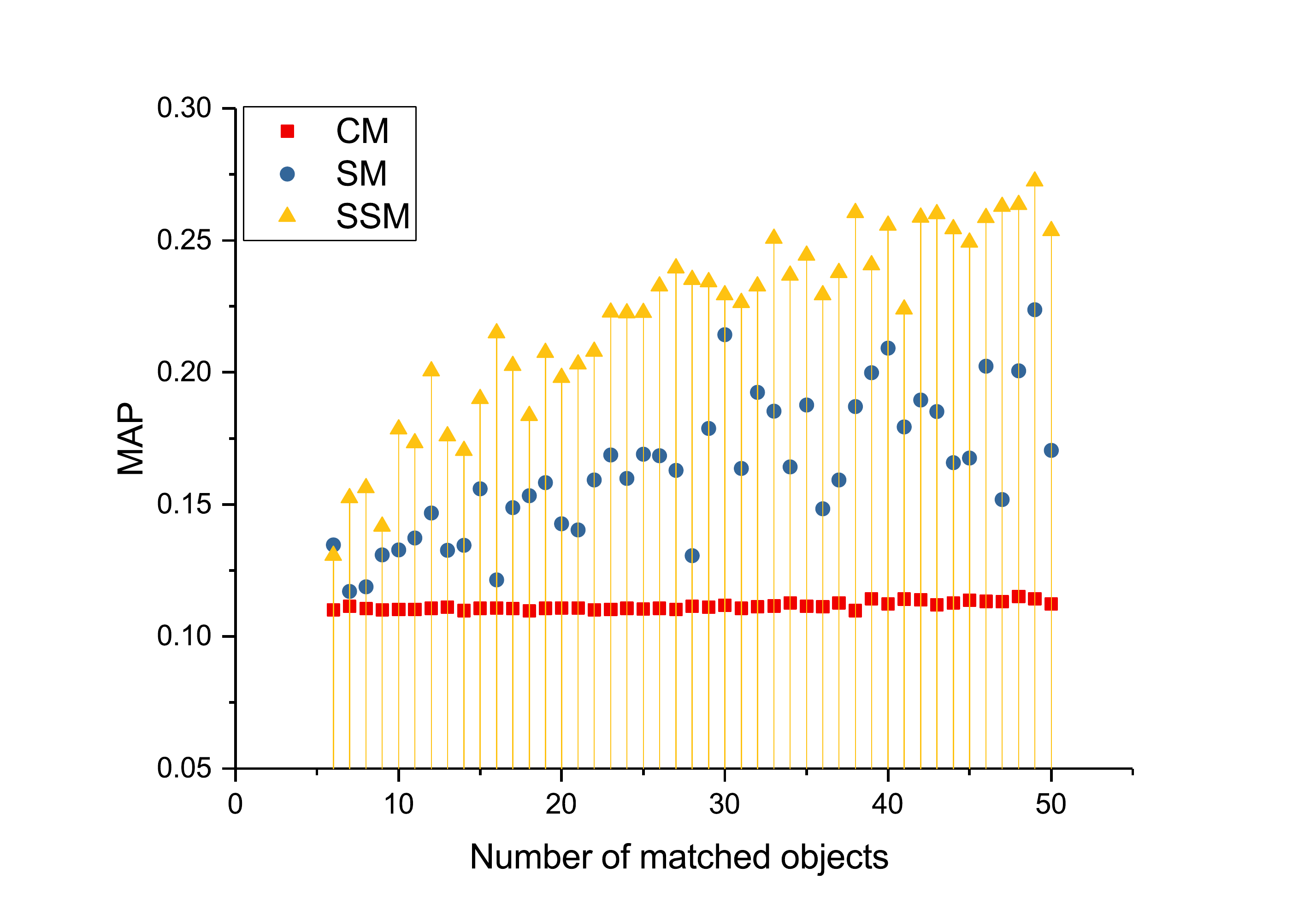}}\\
		(a)&(b)\\
	\end{tabular}
	\caption{The mAP score on Pascal-sentences dataset: (a) for image query, (b) for text query }\label{fig:pascal}
\end{figure*}
\subsection{Evaluation Protocol}
Two common methods for cross-modal coreference resolution are compared with our method(Space Structure Matching, SSM):

\textbf{Correlation Matching (CM)}\cite{Rasiwasia:2010}\cite{Wei:2017}: With canonical correlation analysis(CCA), Rasiwasia proposed to learn a shared space for different modalities where they were maximumly correlated, and projected cross-modal objects into it.

\textbf{Semantic Matching (SM)}\cite{Rasiwasia:2010}\cite{Wei:2017}: SM represent multi-modal data in more abstract levels, then they should be naturally corespondent to each other. Rasiwasia adopted multiclass logistic regression to generate common representations of multi-modal data.

Since our object is to perform coreference resolution with insufficient training data, we take some experiments to evaluate the cross-modal retrieval precision with very limited matched objects. A small subset(ranging 6 to 50) of the dataset is randomly selected as the training set. The trained model on the training set is tested on the left parts.

As in\cite{Rasiwasia:2010}\cite{Wei:2017}\cite{Gao:2018}, \textbf{mean average precision(mAP)} is employed to evaluated the performance of these models, which is a widely used metric of information retrieval\cite{Rasiwasia:2007}:
\begin{equation}
\text{mAP}=\frac{1}{\left| Q \right|}\sum_{i=1}^n{\text{AP}\left( i \right)},
\end{equation}
where $AP(i)$ is the average precision of test sample $i$. The multi-modal objects with same labels are considered semantically similar. For the coreference resolution of $x_i$, the average precision can be computed as:
\begin{equation}
\text{AP}\left( i \right) =\frac{1}{L_i}\sum_{k=1}^n{\text{P}\left( k \right) \delta \left( k \right)},
\end{equation}
where $L_i$ denotes the number of the really similar objects in the target modality, P($k$) is the precision of the results ranked at $k$, $\delta(k)=1$ if the object at rank is similar, $\delta(k)=0$ otherwise. Both the mAP scores of bidirectional queries and their average are computed, and higher mAP indicates better performance of a model.


\subsection{Performance evaluation}

We demonstrate the mAP score of CM, SM and SSM on Pascal-sentences and Wikipedia datasets in Fig. \ref{fig:wiki} and Fig. \ref{fig:pascal}. 

\subsubsection{Results on Wikipedia}
Fig. \ref{fig:wiki} reports the mAP scores on wikipedia dataset in both two retrieval direction: searching similar text for images and the reverse. In Fig. \ref{fig:wiki} (a), the mAP scores of SSM are higher than the others in most cases, and CM's scores are lowest all the time. Both the mAP scores of SSM and SM increase with the number of the matched objects, while that of CM are almost unchanged. The reason of poor performance of CM may be that the training sample is not enough for CCA to find significant correlation between modalities. It should be noticed that the mAP of SM increases rapidly, especially for the image query, and exceeds that of SSM when the size of training data increased to forty.

\subsubsection{Results on Pascal-sentences}
Fig. \ref{fig:pascal} shows the results on Pascal-sentences dataset. The result is pretty similar to that of Wikipedia dataset. In Fig. \ref{fig:pascal} (a), the mAP scores of SSM are also higher than the others in most cases, and CM's scores still keep lowest. Both the mAP scores of SSM and SM increase with the number of the matched objects, and the performance difference demonstrate a decreasing trend. Besides, with the same training set, the mAP scores of text query are higher than that of image query on Pascal-sentences dataset. And the performance of SSM method on Pascal-sentences is better than that on Wikipedia, which could be due to stronger correlation between space structures of Pascal-sentences.

Overall, the proposed SSM method performs better than two popular methods for multi-modal retrieval when training samples are not rich enough. However, the performance have not reached our expectation. The reason may lie on that our method discards the class label information, while it is the most important for results evaluation. There is still much room to improve the method. For example, the way that SM uses class label information may be helpful to improve the precision.

\section{Conclusion}
For cross-modal issues, abstraction and correlation are two useful tools. In this paper, We have demonstrated the possibility to find semantically similar objects from different modality with the correlation between space structures of different modalities, which can be considered as a combination of abstraction and correlation.
 
Firstly, we have proved that the space structures of different modalities are correlated, and employ this kind of correlation for multi-modal coreference resolution. To figure the space structure of each modality more accurately, high level features for image and text are employed. With these above, a semi-supervised method for coreference resolution have been proposed. Different from existing methods, the proposed method mainly utilize a intrinsic correlation between different modalities. It makes our method need pretty few training data to learn the correlation. The experiments on two multi-modal datasets have verified that our method outperformed previous methods when training data are insufficient.

\ifCLASSOPTIONcaptionsoff
  \newpage
\fi
\bibliographystyle{IEEEtran}
\bibliography{CR}

\end{document}